\documentclass[nohyperref]{article}

\usepackage{microtype}
\usepackage{graphicx}
\usepackage{subfigure}
\usepackage{booktabs} 

\usepackage{hyperref}

\usepackage{amsmath,amsfonts,amssymb}
\usepackage{mathtools}
\usepackage{amsthm} 
\usepackage{latexsym}
\usepackage{relsize}

\usepackage[capitalise]{cleveref}
\usepackage{xcolor}
\usepackage{dsfont}

\newcommand{\A}{\mathcal{A}}

\newcommand{\K}{\ensuremath{\mathcal K}}

\def\regret{\mbox{{Regret}}}

\newcommand{\ignore}[1]{}

\theoremstyle{plain}
\newtheorem{theorem}{Theorem}

\newtheorem{assumption}{Assumption}

\newtheorem*{theorem*}{Theorem}
\newtheorem*{lemma*}{Lemma}
\newtheorem*{corollary*}{Corollary}
\newtheorem*{proposition*}{Proposition}
\newtheorem*{claim*}{Claim}
\newtheorem*{fact*}{Fact}
\newtheorem*{observation*}{Observation}
\newtheorem*{assumption*}{Assumption}

\theoremstyle{definition}

\newtheorem{remark}[theorem]{Remark}

\newtheorem*{definition*}{Definition}
\newtheorem*{remark*}{Remark}
\newtheorem*{example*}{Example}

 \theoremstyle{plain}
\newtheorem*{theoremaux}{\theoremauxref}
\gdef\theoremauxref{1}

\DeclareMathAlphabet{\mathbfsf}{\encodingdefault}{\sfdefault}{bx}{n}

\DeclareMathOperator*{\argmin}{arg\,min}

\newcommand{\reals}{\mathbb{R}}

\newcommand{\eqdef}{:=}

\let\oldtfrac\tfrac
\renewcommand{\tfrac}[2]{\smash{\oldtfrac{#1}{#2}}}

\let\nablaold\nabla
\renewcommand{\nabla}{\nablaold\mkern-2.5mu}

\usepackage[accepted]{icml2022}

\usepackage{amsmath}
\usepackage{amssymb}
\usepackage{mathtools}
\usepackage{amsthm}

\usepackage[textsize=tiny]{todonotes}

\icmltitlerunning{Non-convex online learning via algorithmic equivalence}

\begin{document}

\onecolumn
\icmltitle{Non-convex online learning via algorithmic equivalence}



\icmlsetsymbol{equal}{*}

\begin{icmlauthorlist}
\icmlauthor{Udaya Ghai}{pton,goog}
\icmlauthor{Zhou Lu}{pton,goog}
\icmlauthor{Elad Hazan}{pton,goog}

\end{icmlauthorlist}

\icmlaffiliation{pton}{Department of Computer Science, Princeton University, Princeton, NJ}
\icmlaffiliation{goog}{Google AI Princeton, Princeton, NJ}

\icmlcorrespondingauthor{Udaya Ghai}{ughai@cs.princeton.edu}

\icmlkeywords{Machine Learning, ICML}

\vskip 0.3in



\printAffiliationsAndNotice{}  

\begin{abstract}

We study an algorithmic equivalence technique between nonconvex gradient descent and convex mirror descent. We start by looking at a harder problem of regret minimization in online non-convex optimization. We show that under certain geometric and smoothness conditions, online gradient descent applied to  non-convex  functions is an approximation of online mirror descent applied to convex functions under reparameterization. In continuous time, the gradient flow with this reparameterization was shown to be \emph{exactly} equivalent to continuous-time mirror descent by \citet{amid2020reparameterizing}, but theory for the analogous discrete time algorithms is left as an open problem. We prove an $O(T^{\frac{2}{3}})$ regret bound for non-convex online gradient descent in this setting, answering this open problem. Our analysis is based on a new and simple algorithmic equivalence method. 
\end{abstract}

\section{Introduction}
Gradient descent is probably the simplest and most popular algorithm in convex optimization, with numerous and  extensive studies on its convergence properties.  For non-convex objectives, it is known that gradient descent does not necessarily converge to the global optimum, which is in general NP hard. Thus recent  research focuses on alternate objectives, such as finding first-order stationary points efficiently, or the importance of higher order stationary points, e.g. \cite{jin2017escape, lee2016gradient,agarwal2017finding}. 

However, the study of global convergence of gradient descent for non-convex objectives is increasingly important due to the fact that in practice, gradient descent and its variants can achieve zero error on a highly non-convex loss function of a deep neural network.

To explain the success of modern deep learning it is thus important to understand why gradient descent can converge to a global minimum for some highly non-convex functions. Several important research directions have stemmed from this motivation, including the study of optimization for deep linear networks  \cite{ kawaguchi2016deep,laurent2018deep,arora2018optimization}, nonconvex matrix factorization \cite{arora2019implicit}, provable convergence for neural networks in the linearization regime of the neural tangent kernel \cite{soltanolkotabi2018theoretical,du2018gradient_deep,du2018gradient,allenzhu2019convergence,zou2018stochastic,jacot2018neural,bai2019beyond}, and more.

A recent promising approach, proposed in  \cite{amid2020reparameterizing, amid2020winnowing}, is to consider reparameterizing mirror descent as gradient descent. In particular, \cite{amid2020reparameterizing} shows that a continuous-time online mirror descent with a convex loss can be written in an equivalent form of a continuous-time online gradient descent whose loss isn't necessarily convex after reparameterization, but has significant structure. \cite{amid2020winnowing} further explores this idea showing that after reparameterization, online gradient descent has the same worst-case regret bound as the exponentiated gradient algorithm. Both \cite{amid2020reparameterizing} and \cite{amid2020winnowing} provide a new prospect on understanding the convergence of non-convex gradient descent. 

\subsection{The Amid-Warmuth Question}\label{sec:amid-warmuth}

The works of \cite{amid2020reparameterizing} \cite{amid2020winnowing} are limited in two respects: either they are restricted to the continuous-time gradient-flow setting, which is impractical, or they address specific algorithms (Winnow, Hedge, EG for linear regression) with a relative entropy divergence. Amid and Warmuth pose the question of extending the reparameterization idea to the general online convex optimization as well as the conditions where this can be obtained. Solving the question posed in \cite{amid2020reparameterizing}  would give an answer to realistic scenarios and algorithms, with precise regret and running time bounds, for certain non-convex optimization problems that can be solved to global optimality using gradient descent. 

\subsection{Statement of Results}

Our study starts from the more general problem of regret minimization in online convex optimization. Specifically, we study the reparameterization of online mirror descent (OMD) as online gradient descent (OGD) \cite{OGD} for general online convex optimization. We show that under certain geometric and smoothness conditions, a non-convex reparameterized OGD algorithm closely match the original convex OMD algorithm. Quantitative bounds on this matching allow us to  prove that the OGD algorithm after reparameterization achieves an $O(T^{\frac{2}{3}})$ regret bound. This bound is discrete-time and fairly general, applying broadly to smooth, bounded mirror descent regularizers rather than anything more specific. We also provide sufficient conditions for nonconvex losses such that OGD implicitly has an equivalent mirror descent regularizer and hence will have $O(T^{\frac{2}{3}})$ regret. 

Our analysis relies on a simple and new algorithmic equivalence technique, which may be of independent interest. The key step is to show that, from the same starting point, the outputs of OMD and reparameterized OGD are very close to each other. Considering the OGD update as a perturbed version of OMD, and combining with the fact that OMD can tolerate bounded (adversarial) perturbation per trial, allows us to prove an $O(T^{\frac{2}{3}})$ regret bound for the OGD update.

This answers the open question raised in \cite{amid2020reparameterizing}, generalizing from continuous-time gradient flow and specialized algorithms, to general online convex optimization. The regularization term of MD is also generalized from relative entropy to any strongly convex function. Further, for discrete-time reparameterization, our result extend to arbitrary Lipschitz convex loss functions, as opposed to custom analysis for specific settings of expert prediction, linear prediction, and linear regression \cite{amid2020winnowing}. 


\subsection{Related Works}
Gradient descent is the simplest and one of the most popular algorithms in convex optimization. Extensive studies have been conducted to show the convergence of GD to the global minimizer in both the stochastic optimization setting \cite{nemirovski2009robust}, and the online convex optimization setting \cite{zinkevich2003online}. For the non-convex setting, however, finding the global minimum is NP-hard and the study mostly focuses on the convergence to first-order stationary points instead. The well-known descent lemma guarantees the convergence of non-convex GD when the objective is smooth, and dropping smoothness is possible if other conditions are introduced \cite{bauschke2017descent}. Many works consider avoiding saddle points, \cite{lee2016gradient} shows that GD converges to local minimizers a.s. if all saddle points are strict, \cite{jin2017escape} proposes perturbed gradient descent to escape saddle points. It's also shown that even in the online setting, chasing first-order stationary points is possible \cite{hazan2017efficient}.

Mirror descent, first introduced by \cite{nemirovskij1983problem}, generalizes the gradient descent method in the sense that it can adapt to the 'geometry' of the optimization problem \cite{beck2003mirror}. Its analogue in the online setting, online mirror descent also achieves tight regret bounds \cite{srebro2011universality} and without projection is shown to be equivalent to the classical regularized follow-the-leader (RFTL) method. Though it's natural to think of mirror descent as changing the geometry of the optimization objective in gradient descent \cite{gunasekar2021mirrorless}, the idea to explain non-convex GD by a corresponding convex MD is explored only recently.

The most relevant works to ours are \cite{amid2020reparameterizing} and \cite{amid2020winnowing}, which prove the equivalence between continuous time OMD and OGD after reparameterization, and special discrete-time algorithms when the regularization is the relative entropy, respectively. Our results extend \cite{amid2020reparameterizing} and \cite{amid2020winnowing} to general discrete-time OCO with general regularization.

\section{Preliminaries}\label{sec:prelim}
\subsection{Notation}
For a function $f: \reals^d \rightarrow \reals^d$, we use the notation $J_f(x): \reals^d \rightarrow \reals^d$ to denote the Jacobian of $f$ at $x$. Given a strictly convex function $R:\reals^d \rightarrow \reals^d$, we denote the Bregman divergence as $D_R(x\| y) \eqdef R(x) - R(y) - \nabla R(y)^{\top} (x-y)$.  For a convex set $\K$ and strictly convex regularizer $R$, we use $\Pi^R_{\K}(x) \eqdef \argmin_{y \in \K} D_R(y\| x)$ to denote Bregman projection, with  shorthand $\Pi_{\K} \eqdef \Pi^{\|\cdot \|^2}_{\K}$ for euclidean projection. 

\subsection{Online Convex Optimization}
We consider the online convex optimization (OCO) problem. At each round $t$ the player $\A$ chooses $x_t\in \K$ where $\K\subset \reals^d$ is some convex domain, then the adversary reveals loss function $f_t(x)$ and player suffers loss $f_t(x_t)$. The goal is to minimize regret:
$$
\regret(\A)=\sum_{t=1}^T f_t(x_t)-\min_{x\in \K} \sum_{t=1}^Tf_t(x)
$$
The player is allowed to get access to the (sub-)gradient $\nabla_x f_t(x_t)$ as well.

\subsection{Reparameterizing continuous-time mirror flow as gradient flow} 
In the continuous-time setting, we have exact conditions from \cite{amid2020reparameterizing} where the trajectory of mirror-flow \emph{exactly} coincides with that of gradient descent in a reparameterized space. Concretely, mirror flow on $f:\reals^d \rightarrow \reals$ is defined by the following ordinary differential equation (ODE):
\begin{align}\label{eq:mirror_flow}
    \frac{\partial}{\partial t} \nabla R(x(t)) = -\eta \nabla f(x(t))~. 
\end{align}

In Theorem~2 of \cite{amid2020reparameterizing}, this is shown to be equivalent to the gradient flow ODE with $x(t) = q(u(t))$:
\begin{align}\label{eq:reparam_flow}
    \frac{\partial u}{\partial t} = -\eta \nabla f(q(u(t)))~,
\end{align}
if the mirror descent regularizer $R$ and reparameterization function $q$ satisfy $[\nabla_R^{2}(q(u))]^{-1}=J_q(u) J_q(u)^{\top}$. This relationship assures that, up to second order factors in $\|u-v\|_2$
\begin{align} \label{eq:approx_div}
    D_R(q(u)||q(v)) \approx \frac{1}{2}\|u-v\|^2_2 ~,
\end{align}
so the geometry induced by $R$ is approximately transformed into a Euclidean geometry by $q^{-1}$.  Because higher order factors vanish in continuous time, this assures that mirror flow and this reparameterized gradient flow coincide as desired.

\subsection{Discretizing the updates}
Unfortunately, continuous-time updates cannot be implemented in practice so we must rely on discretization.  Using the forward Euler scheme for discretization of \eqref{eq:mirror_flow}, yields the mirror descent algorithm.  After discretization, the exact equivalence of mirror descent and reparameterized gradient descent no longer holds as higher order factors in \eqref{eq:approx_div} are relevant. This motivates the open problem of \cite{amid2020reparameterizing} of finding general conditions under which the discretizations of these continuous time trajectories closely track each other. We tackle extending these results through the online setting described above.

The Online Mirror Descent (OMD) algorithm has the following update on $x_t$:
\begin{align*}
    &\nabla R (y_{t+1})=\nabla R(x_t)-\eta \nabla f_t(x_t)\\
     &x_{t+1}=\argmin_{x\in \K} D_R(x\| y_{t+1})
\end{align*}
where $R: \K \rightarrow \reals$ is a 1-strongly convex regularization over the domain\footnote{Usually, the strong convexity of a mirror descent regularizer is with respect to some norm that is not necessarily the $\ell_2$ norm. In this work, the focus is on the dependence on the regret on $T$, so for simplicity we consider just the $\ell_2$ strong convexity, which will worsen constants, but does not affect the dependence on $T$.}, $D_R$ is the Bregman divergence, and $y_1$ is initialized to satisfy $\nabla R (y_1)=0$. 

Another equivalent interpretation of OMD is
$$
x_{t+1}=\argmin_{x\in \K} \nabla f_t(x_t)^{\top}(x-x_t) +\frac{1}{\eta} D_R(x||x_t)~.
$$
The most important special case is  online gradient descent (OGD)
$
x_{t+1}=\Pi_{\K} (x_t-\eta \nabla f_t(x_t))
$
, which is OMD with $D_R(x||x_t)=\frac{1}{2}\|x-x_t\|_2^2$. In this paper we consider reparameterizing general OMD update as a simple OGD update. We introduce assumptions of the regularizer, reparameterization, and losses in the sequel.

\subsection{Assumptions}
\begin{assumption}\label{assumption:reparam_jacobian}
There exists a convex domain $\K'\subset \reals^{d}$ and a bijective reparameterization function $q: \K'\to \K$ satisfying 
$
[\nabla^{2} R(q(u))]^{-1}=J_q(u) J_q(u)^{\top}~.
$
\end{assumption}

We also make the following smoothness/Lipschitz assumptions:
\begin{assumption} \label{assumption:lipschitz}
Let $G > 1$ be a constant. We assume $q$ is $G$-Lipschitz, and the regularization $R$ is smooth with its 1-st and 3-rd derivatives upper bounded by some constant $G$. The 1-st and 2-nd derivatives of $q^{-1}$ are bounded by $G$.  Furthermore, assume that for all $x \in \K$ $D_R(x \| \cdot)$ is $G$-Lipschitz over $\K$.
\end{assumption}

Finally, we assume loss functions have bounded gradients and the convex domain has bounded diameter with respect to the Bregman divergence. 

\begin{assumption}\label{assumption:bounds}
The loss functions $f_t$ have gradients bounded by $\|\nabla f_t(x)\|_2 \le G_F$ for all $x\in \K$. The diameter of $\K$, $ \sup_{x,y \in \K} D_R(x\| y) \le D$~.
\end{assumption}

\section{Algorithm}\label{sec:result}
We now provide our main result, a regret bound for Algorithm~\ref{alg1}. We note, Algorithm~\ref{alg2} and Algorithm~\ref{alg1}  are the (projected) forward euler discretizations of \eqref{eq:mirror_flow} and \eqref{eq:reparam_flow}. It's tempting to directly analyze the regret of OGD with losses $\tilde{f}_t(u)=f_t(q(u))$. The main barrier in doing so is that $\tilde{f}$ isn't necessarily convex. However, we show that the OMD and OGD updates are in fact $O(\eta^{\frac{3}{2}})$-close to each other, therefore the OGD update still suffers only sublinear regret.

\begin{minipage}{0.49\linewidth}
\begin{algorithm}[H]
\caption{Online Mirror Descent} \label{alg2}
\begin{algorithmic}[1]
\STATE Input: Initialization $x_1 \in \K$, regularizer $R$.\\
\FOR{$t = 1, \ldots, T$}
\STATE Predict $x_t$
\STATE Observe $\nabla f_t (x_t)$ 
\STATE Update 
\begin{align*}
    y_{t+1} &=  (\nabla R)^{-1}(\nabla R(x_t) - \eta \nabla f_t)\\
    x_{t+1} &= \Pi^{R}_{\K} (y_{t+1})
\end{align*}
\ENDFOR
\end{algorithmic}
\end{algorithm}
\end{minipage}
\hfill
\begin{minipage}{0.49\linewidth}
\begin{algorithm}[H]
\caption{Online Gradient Descent} \label{alg1}
\begin{algorithmic}[1]
\STATE Input: Initialization $u_1 \in \K' = q^{-1}(\K)$.\\
\FOR{$t = 1, \ldots, T$}
\STATE Predict $u_t$
\STATE Observe $\nabla \tilde{f}_t(u_t) = \nabla f(q(u_t)))$
\STATE 
Update 
\begin{align*}
    v_{t+1} &= u_t-\eta \nabla \tilde{f}_t(u_t)\\
    u_{t+1} &= \Pi_{\K'} (v_{t+1})
\end{align*}

\ENDFOR
\end{algorithmic}
\end{algorithm}
\end{minipage}

\begin{remark}
The reparameterized gradient descent uses euclidean projection rather than Bregman projection.  This can potentially be an advantage if euclidean projection cab be computed efficiently on $q^{-1}(\K)$ (eg. via method like \cite{usmanova2021fast}), while Bregman projection is less efficient on $\K$.
\end{remark}

\begin{theorem}\label{thm:main}
Under Assumptions~\ref{assumption:reparam_jacobian}-\ref{assumption:bounds}, by setting $\eta=T^{-2/3} D^{2/3} G^{-10/3} G^{-1}_F$ the regret of Algorithm \ref{alg1} is upper bounded by $O(T^{2/3} D^{1/3} G^{10/3}G_F)$.
\end{theorem}
Proof of the theorem and examples can be found in the full paper \cite{fullpaper}.

\bibliography{bib}
\bibliographystyle{icml2022}

\end{document}